\documentclass[11pt,reqno, english]{amsart}
\usepackage{amsmath,amssymb,amscd,amsthm,amsfonts}
\usepackage{graphicx,subcaption}
\usepackage{hyperref}
\usepackage{dsfont}
\usepackage{xcolor}
\usepackage{color,soul}
\usepackage[utf8]{inputenc}
\usepackage[nobysame, alphabetic]{amsrefs}
\usepackage{tikz}
\usepackage[capitalise]{cleveref}
\usepackage{wrapfig}
\usepackage{floatrow}
\usepackage{algorithm}
\usepackage[noend]{algpseudocode}

\newtheorem{theorem}{Theorem}[subsection]

\newtheorem{lemma}[theorem]{Lemma}

\newtheorem{definition}{Definition}

\newcommand{\rr}{\mathds{R}}

\DeclareMathOperator{\conv}{conv}
\DeclareMathOperator{\dist}{dist}
\title{Tverberg's theorem and multi-class support vector machines}

\author[Sober\'on]{Pablo Sober\'on}\address{Baruch College \& The Graduate Center, City University of New York, New York, NY 10010} 
\email{psoberon@gc.cuny.edu}

\thanks{he research of P. Sober\'on was supported by NSF CAREER award no. 2237324, NSF award no. 2054419 and a PSC-CUNY Trad B award.}

\subjclass{52A35, 52C35, 62R07, 62R40}
\keywords{
  Support vector machine, Tverberg's theorem}

\begin{document}

\begin{abstract}
    We show how, using linear-algebraic tools developed to prove Tverberg's theorem in combinatorial geometry, we can design new models of multi-class support vector machines (SVMs).  These supervised learning protocols require fewer conditions to classify sets of points, and can be computed using existing binary SVM algorithms in higher-dimensional spaces, including soft-margin SVM algorithms. We describe how the theoretical guarantees of standard support vector machines transfer to these new classes of multi-class support vector machines.  We give a new simple proof of a geometric characterization of support vectors for largest margin SVMs by Veelaert.
\end{abstract}

\maketitle

\section{Introduction}

Support vector machines (SVMs) are a supervised learning model for data classification with a wide range of applications \cites{Boser1992, hearst1998support, steinwart2008support}.  The underlying geometric problem is, given two finite sets $A, B$ of points in $\rr^d$, to find a hyperplane separating $A$ and $B$.  A key example are largest margin SVMs, in which the separating hyperplane maximizes the minimum distance to each set.  We assume that $\conv A \cap \conv B = \emptyset$ for such a hyperplane to exist.  If $\conv A \cap \conv B \neq \emptyset$, minimizing the number of misclassified points by a hyperplane is NP-hard, but one can use adaptations such as soft-margin SVM.

A variation of this model, multi-class support vector machines, arises when we want to classify more than two sets of points.  If we want to classify $k$ classes $A_1, \ldots, A_k$, the most common approaches are one-versus-all (1vA) and all-versus-all (AvA) models, both of which break the classification problem into many binary support vector machines.  In the first, we have to solve for $k$ support vector machines, each separating a single class from the union of the other $k-1$,  In the second, we solve for $\binom{k}{2}$ support vector machines separating each pair of classes.  Some optimization methods aggregate several SVMs into a single optimization problem in a higher-dimensional space, which can then be adjusted to be easier to solve \cites{Franc.2002, crammer2001algorithmic}.  Multiple other models for multi-class SVMs have been proposed \cites{duan2005best, hsu2002comparison}.

Many combinatorial properties of SVMs are related to classic results in discrete geometry, such as Radon's theorem \cites{Veelaert.2015, Adams2022}.  Radon's theorem states that \textit{given $d+2$ points in $\rr^d$, there exists a partition of them into two sets whose convex hulls intersect} \cites{Radon:1921vh}.  A well-known generalization of Radon's theorem is Tverberg's theorem, in which we seek to split a set of points into several subsets whose convex hulls intersect.  Tverberg proved that \textit{given $(k-1)(d+1)+1$ points in $\rr^d$, there exists a partition of them into $k$ sets whose convex hulls intersect} \cites{Tverberg:1966tb}.  The case $k=2$ is Radon's theorem.  There is active research around Tverberg's theorem, as it has led to important developments in discrete geometry and topological combinatorics \cites{Barany:2016vx, Blagojevic:2017bl, Barany:2018fy}.

A far-reaching tool to prove variations of Tverberg's theorem is a linear-algebraic technique devised by Sarkaria \cite{Sarkaria:1992vt} and simplified by B\'ar\'any and Onn \cite{Barany:1996bz}.  In addition to leading to one of the simplest known proofs of Tverberg's theorem, this technique is highly malleable and can be used to prove a multitude of variations of Tverberg's theorem.

The goal of this manuscript is to show a link between Tverberg's theorem and multi-class SVMs via the linear algebra techniques mentioned above.  The existence of a connection between multi-class SVMs and Tverberg's theorem was conjectured by Adams et al. \cites{Adams2022}, when they linked Radon's theorem to binary SVMs.  To have a multi-class SVM that does not missclassify any points, the (1vA) model requires each class to be separable from the union of the other $k-1$, and the (AvA) model requires any two classes to be separable.  We propose a new type of multi-class SVM which uses a weaker condition.   Applying our model for $k=2$ leads to classic SVMs.  Of course, since we do not ask any two $A_i$, $A_j$ to be separable, potential miss-classifications are unavoidable.  We only require $\bigcap_{i=1}^k \conv(A_i) = \emptyset$.  The output will be a family of $k$ closed half-spaces $H_1, \ldots, H_k$ such that 
\begin{itemize}
    \item For each $i=1,\ldots, k$ we have have $A_i \subset H_i$ and
    \item the half-spaces satisfy $\bigcap_{i=1}^k H_i = \emptyset.$
\end{itemize}

\begin{figure}
    \centerline{\includegraphics[width=\textwidth]{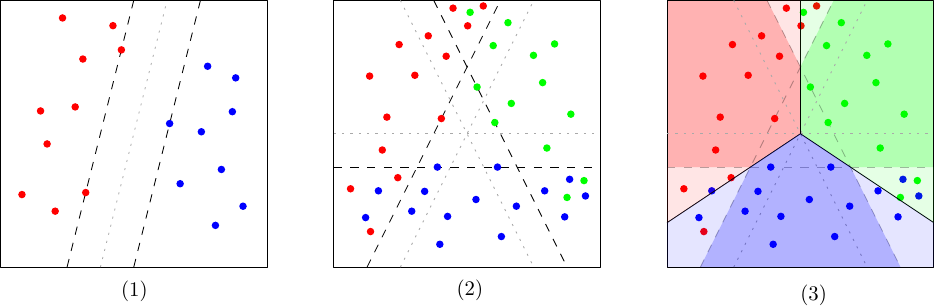}}
    \caption{(1) An example of an SVM, we emphasize the support hyperplanes parallel to the generated hyperplane for each class. (2) An example of an multi-class SVM under the proposed model.  Notice that it is not possible to separate any two classes of points with a hyperplane.  (3) The half-spaces in part (2) can be used to classify space using convex regions.  The model can distinguish regions where it is ambiguous.}
    \label{fig:example-full}
\end{figure}

We describe how the half-spaces can be used to split $\rr^d$ into $k$ convex regions, each corresponding to an $A_i$.  The model can also be used to distinguish the regions of ambiguity.  The subdivision of $\rr^d$ is most natural when $k \le d+1$.

We describe in \cref{tab:comparison} the complexity of computing these multi-class SVM.  We compare directly with the complexity of computing a single SVM, to highlight the influence of the dimension.

\begin{table}[]
    \centering
$\begin{tabular}{|c|c|}
\hline
     (AvA) & $\binom{k}{2} \tau (n/k,n/k; d)$  \\
     \hline
     (1vA) & $k \cdot \tau (n/k,n-(n/k);d)$ \\
     \hline
     (Simple TSVM) & $\tau(1,n-1;(d+1)(k-1))$ \\
     \hline
     (TSVM) & $O\Big(n \cdot \tau \Big(1,(d+1)(k-1)+1; (d+1)(k-1)\Big)\Big)$ (randomized) \\
      & $\tau((n/k)^k,1; d(k-1))$ (deterministic) \\
      \hline
\end{tabular}$
    \caption{Any linear SVM algorithm can be applied to the computation of our multi-class SVM, including soft-margin SVMs.  If we denote $\tau(a,b;d)$ the complexity of computing an SVM with $a+b$ data points in $\rr^d$ (one class with $a$ points and one with $b$ points), then the computational complexity in terms of $n$ of our results can be described as listed above.  We assume our original set o points has $k$ classes with $n/k$ points each.  We include the complexity of a naive approach to (AvA) and (1vA) for comparison.  The randomized algorithm for (TSVM) also has constant factors that depend on the product $dk$ but not $n$.\\
    Statistical guarantees would be the same as those running a linear SVM with the parameters above.  (Simple TSVM) and deterministic (TSVM) are equivalent to running a single SVM, while (TSVM), while randomized (TSVM) is equivalent to running $O(n)$ binary SVMs.}
    \label{tab:comparison}
\end{table}


Tverberg's theorem is a challenging algorithmic problem \cites{har2020journey}.  One key difference between the problem addressed in this manuscript and the problem of finding Tverbreg partitions is that when training SVMs the labels are assigned before-hand.  Tverberg's theorem has also been applied to multi-class logistic regression \cites{de2020stochastic}.  

Since the constructions are based on Sarkaria's linear-algebraic technique, we can deduce several combinatorial properties of these multi-class SVMs.  The model (simple TSVM) is invariant under orthogonal transformations, but not under translations.  The model (TSVM) is invariant under any isometry of $\rr^d$.  To prove these properties, a closer look at Sarkaria's method is needed, so the arguments presented here may be useful in the classic context of variations of Tverberg's theorem.

We also discuss the existence and properties of support vectors.  It is known that for any two separable sets $A, B \subset \rr^d$, there exist $A' \subset A, B' \subset B$ such that $|A' \cup B'| \le d+1$ and such that the largest-margin SVM induced by $A, B$ is the same as the largest-margin SVM induced by $A', B'$.  For (TSVM) and (simple TSVM) a similar property holds.  For any $k$-tuples of sets $A_1, \ldots, A_k$, there is a $(k-1)(d+1)$-subset of $A_1 \cup \ldots \cup A_k$ that induces the same (TSVM).  The same holds for (simple TSVM).  In either case we call this $(k-1)(d+1)$-subset the \textit{support vectors} of the multi-class SVM.

The manuscript is organized as follows.  First, we present in \cref{sec:projection} a new proof of a characterization of critical points in largest-margin SVMs.  In \cref{sec:tools} we describe the linear-algebraic tools needed for our constructions.  In \cref{sec:construction} we describe the models (TSVM) and (simple TSVM), and their main properties.  In \cref{sec:subdivision} we discuss the induced partitions of $\rr^d$ and finally in \cref{sec:equivariance} we study how the model behaves when we apply orthogonal transformations to the sets of points.


\section{Projection of support vectors.}\label{sec:projection}

Given two finite sets of points $A, B$ in $\rr^d$ that are linearly separable, let $H$ be a separating hyperplane at maximal distance from $A$ and $B$.  We denote by $\varepsilon$ this distance, so
\[
\varepsilon = \dist(H, \conv A) = \dist (H, \conv B)
\]

We say that a point in $a \in A$ is a support vector if $\dist(a, H) = \varepsilon$, and similarly for points in $B$.  We assign labels to the points so that the points of $A$ assigned \textit{positive} and the points of $B$ are assigned \textit{negative}.

One interesting property about the projections of the support vectors is that the convex hulls of the projections onto the separating hyperplane of each side intersect, as in \cref{fig:projection}.  This was proven independently by Veelaert and by Adams, Carr, and Farnell \cites{Veelaert.2015, Adams2022}.  One of the proofs involves the Karush--Kuhn--Tucker theorem and the other Householder transformations.  We present an elementary proof.

\begin{theorem}
Given a separable set of points in $\rr^d$ with two labels, the convex hulls of the projections of the negative and positive support vectors onto the induced largest margin SVM intersect.
\end{theorem}

\begin{figure}
    \centering
    \includegraphics{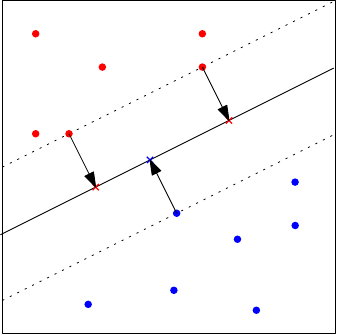}
    \caption{An example of a largest-margin SVM with two sets.  If we project the support vectors onto the separating hyperplane, the convex hulls of the projections of different sides intersect.}
    \label{fig:projection}
\end{figure}

\begin{proof}
Let $L$ be the set of labeled points.  Let $H$ be the separating hyperplane at maximum distance from the labeled sets, and let $S_+$, $S_-$ be the positive and negative support vectors, respectively.  Let $\varepsilon > 0$ be the distance of the support vectors to $H$.
This means that 
\begin{align*}
   \dist(x,H)  & = \varepsilon \qquad \mbox{for $x \in S_+ \cup S_-$} \\
\dist(x,H) & > \varepsilon \qquad \mbox{for $x \in L\setminus(S_+ \cup S_-)$}.
\end{align*}

Let $P_+$ be the orthogonal projection of $S_+$ onto $H$, and $P_-$ be the orthogonal projections of $S_-$ onto $H$.  We assume that $\conv P_+ \cap \conv P_- = \emptyset$ and look for a contradiction.

Since $\conv P_+ \cap \conv P_- = \emptyset$, there exists a co-dimension one affine subspace $H'$ of $H$ that separates $P_+$ and $P_-$.  Notice that $H'$ is a co-dimension two affine subspace of $\rr^d$.

Let us project all the points into the two-dimensional subspace $(H')^{\perp}$.  We denote this projection by $\pi$.  In $(H')^{\perp}$, $\pi(H)$ is a line $\ell$ and $\pi(H')$ is a point $p$.  Let $\ell_2$ be the orthogonal line to $\ell$ through $p$.

The lines $\ell$ and $\ell_2$ split $(H')^{\perp}$ into four quadrants.  Since $\conv P_+ \cap \conv P_- = \emptyset$, the points of $\pi(S_+)$ and those of $\pi(S_-)$ are separated by $\ell_2$.  They are also separated by $\ell$ by construction, so $\pi(S_+)$ and $\pi(S_-)$ are in opposite quadrants.

This means that we can rotate $\ell$ slightly around $p$ so that its distance to each point in $\pi(S_+)$ and $\pi(S_-)$ increases.  If the angle of rotation is small enough, the distance of $\pi^{-1}(\ell)$ to the rest of the points in $L$ remains strictly larger than $\varepsilon$.  This contradicts $H$ being the largest-margin SVM. 
\end{proof}

\section{Linear-algebraic tools}\label{sec:tools}

In this section, we introduce Sarkaria's construction to tackle Tverberg-type problems.  Suppose we are given $k$ sets $A_1, \ldots, A_k$ in $\rr^d$.  We introduce $v_1, \ldots, v_k$, which are the vertices of a regular simplex in $\rr^{k-1}$ centered at the origin.  We further assume that each $v_i$ is a unit vector.  A crucial property of this $k$-tuple is that its linear dependences are precisely the linear combinations in which all coefficients are equal.   For each $i$ we associate $v_i$ to $A_i$.  Given $x \in A_i$ we first append a coordinate $1$ to make it into a vector in $\rr^{d+1}$, $\bar{x} = \begin{bmatrix} x \\ 1\end{bmatrix}$.  Then, we take the tensor product with its corresponding $v_i$, defining  ${S(x) = \bar{x} \otimes v_i = \bar{x} v_i^T \in \rr^{(d+1)(k-1)}.}$

In this manuscript we treat $\rr^{(d+1)(k-1)}$ as the set of $(d+1)(k-1)\times 1$ vectors and as the set of $(d+1)\times(k-1)$ matrices interchangeably.

Finally, for $i=1,\ldots, k$, we define the set
\[
Y_i = \{S(x) : x \in A_i, \ i=1,\ldots, k\} \subset \rr^{(d+1)(k-1)}
\]

The main difference with our approach and the one by  B\'ar\'any and Onn is that for each point in $A_1 \cup \ldots A_k$ we already know to which class it belongs, so it yields a unique point in the higher-dimensional space.  When one wants to prove Tverberg's theorem, we have to assign classes to unlabeled sets of points, so each point in $\rr^d$ is represented by a $k$-tuple in $\rr^{(d+1)(k-1)}$.

The main reason why this transformation can be used to study Tverberg-type problems and why we can use it in the context of multi-class SVMs is the following lemma.

\begin{lemma}\label{lemma:sarkaria}
Let $k, d$ be positive integers.  Let $A_1, \ldots, A_k$ be finite sets of points in $\rr^d$ such that $\bigcap_{i=1}^k \conv (A_i) = \emptyset$.  Then, for $Y_1, \ldots, Y_k$ defined as above, $0 \not\in \conv\left( \bigcup_{i=1}^k Y_i\right)$.
\end{lemma}

\begin{proof}
Let $A = \bigcup_{i=1}^k A_i$ and $Y = \bigcup_{i=1}^k Y_i$.  We prove the contrapositive.  Assume that the origin in $\rr^{(d+1)(k-1)}$ is in the convex hull of $Y$.  We want to show that the convex hulls of the sets $A_i$ intersect.  Then, for each $x \in A$ there exists a non-negative coefficient $\alpha(x)$ such that $\sum_{x \in A}\alpha(x) = 1$ and
\begin{align*}
0 = \sum_{x \in A} \alpha(x) S(x) & = \sum_{x \in A_1} \alpha(x) S(x) + \ldots +  \sum_{x \in A_k} \alpha(x) S(x) \\
& = \sum_{x \in A_1} \alpha(x) (\bar{x} \otimes v_1) + \ldots +  \sum_{x \in A_K} \alpha(x) (\bar{x} \otimes v_k)    \\
& = \left( \sum_{x \in A_1} \alpha(x) \bar{x}\right) \otimes v_1 + \ldots +  \left( \sum_{x \in A_k} \alpha(x) \bar{x}\right) \otimes v_k.
\end{align*}

If we look at the linear dependences of $v_1, \ldots, v_k$ in $\rr^{k-1}$, we can see that $\beta_1 v_1 + \ldots + \beta_k v_k = 0$ if an only if $\beta_1 = \ldots = \beta_k$.  This carries through the tensor product and we have ${\sum_{x \in A_1} \alpha(x) \bar{x} = \ldots = \sum_{x \in A_k} \alpha(x) \bar{x}.}$

This is an equality in $\rr^{d+1}$.  If we look at the last coordinate, we have
$\sum_{x \in A_1} \alpha(x) = \ldots = \sum_{x \in A_k} \alpha(x).$
Since the total sum of the coefficients was one, each of the sums above must be $1/k$.  If we look at the first $d$ coordinates and multiply each equation by $k$, we have
\[
\sum_{x \in A_1} (k\alpha(x))x = \ldots = \sum_{x \in A_k} (k\alpha(x))x.
\]
and each of the terms above is a convex combination.  This means that the convex hulls of the $A_i$ have non-empty intersection, $\bigcap_{i=1}^k \conv (A_i) \neq \emptyset$.
\end{proof}

Therefore, the origin in $\rr^{(d+1)(k-1)}$ can be separated from $Y$ by a hyperplane.  We can find this hyperplane with any existing algorithm for SVMs, which is the central point of this manuscript.  If we have a hyperplane separating a set from the origin in $\rr^{(d+1)(k-1)}$, we want to obtain a set of half-spaces as described in the introduction.

In other words, we need to be able to map hyperplanes in $\rr^{(d+1)(k-1)}$ into $k$-tuples of hyperplanes in $\rr^d$ explicitly.  This has been done recently \cites{Sarkar:2020uk}.  We describe the process below.

Let $\Pi: \rr^{d+1} \to \rr^d$ be the function that erases the last coordinate.  For each $y \in \rr^{(d+1)(k-1)}$ let us think of $y$ as a $(d+1) \times (k-1)$ matrix.  For $i \in [k]$ we define the function
\begin{align*}
    f_i : \rr^{(d+1)(k-1)} & \to \rr^d \\
    y & \mapsto \Pi (y v_i) 
\end{align*}
where $yv_i$ is considered as a product of matrices.

\begin{lemma}
If $x \in A_i$, then $f_i(S(x))=x$.
\end{lemma}

\begin{proof}
A simple computation shows that
\[
f_i(S(x)) = f_i(\bar{x} v_i^T) = \Pi (\bar{x} v_i^T v_i) = \Pi (\bar{x}) = x.
\]
The third equality follows since $v_i$ is a unit vector.
\end{proof}

For each $i \in [k]$, consider the $d$-dimensional affine subspace $U_i = \{\bar{x} \otimes v_i : x \in \rr^d\} \subset \rr^{(d+1)(k-1)}$.  Given a half-space $H$ in $\rr^{(d+1)(k-1)}$, consider the $k$ half-spaces in $\rr^d$
defined by $H_i = f_i (U_i \cap H)$ for $i\in [k]$.

\begin{lemma}
Let $H$ be a closed half-space in $\rr^{(k-1)(d+1)}$. If $0 \not\in H $ then $\bigcap_{i=1}^k H_i = \emptyset$.
\end{lemma}

\begin{proof}
As before, let's consider $\rr^{(d+1)(k-1)}$ as the set of $(d+1) \times (k-1)$ matrices.  Each closed half-space $H$ can be defined using a linear functional and a constant.  Using the Frobenius product, we can express $H$ using a $(k-1)\times (d+1)$ matrix $M$ and a constant $\lambda$ such that
\[
H = \{S \in \rr^{(k-1)(d+1)} : \operatorname{tr}(S M) \ge \lambda\}.
\]
Since the origin is not contained in $H$, we can assume that $\lambda > 0$.  Suppose on the contrary that there exists an $x \in \rr^d$ so that $x \in \bigcap_{i=1}^k f_i (U_i \cap H)$ and we look for a contradiction.  In other words, for $i=1,\ldots, k$ we have $\bar{x} \otimes v_i \in H$, so
$\operatorname{tr} (\bar{x}v_i^T M) \ge \lambda > 0$.

If we write each of the $k$ inequalities as $i$ varies and add them, we have
\[
0< k \lambda \le \sum_{i=1}^k \operatorname{tr} (\bar{x}v_i^T M) = \operatorname{tr} \left(\bar{x}\left({\sum_{i=1}^k v_i^T}\right) M\right) = 0.
\]
The last equality follows as $\sum_{i=1}^k v_i = 0$.  This is the contradiction we wanted.
\end{proof}

Now we have all the ingredients to define a multi-class SVM.  Another important subspace for our computations is the following $d(k-1)$-dimensional space
\[
R = \{ y \in \rr^{(d+1)(k-1)}: \mbox{the last row of $y$, as a $(d+1)\times (k-1)$ matrix, is zero} \}.
\]

This subspace has been used previously to prove some variations of Tverberg's theorem with some coloring conditions added to the set \cites{soberon2015equal}.  Some particular translates of $R$ will also be useful.  For $i=1,\ldots, k$ we define $R_i = \{ S \in \rr^{(k-1)(d+1)}: \mbox{the last row of $S$ is $v_i^T$} \}.$

Notice that $U_i \subset R_i$ for each $i=1,\ldots, k$.

\begin{lemma}\label{lem:barycenter}
Let $z_1, z_2, \ldots, z_k \in \rr^{(d+1)(k-1)}$ such that $z_i \in U_i$ for each $i=1,\ldots, k$.  The only point in $R \cap \conv\{z_1,\ldots, z_k\}$ is the barycenter of the set $\{z_1, \ldots, z_k\}$.
\end{lemma}

\begin{proof}
Consider each $z_i$ as a $(d+1)\times (k-1)$ matrix.  Suppose that $\lambda_1z_1 + \ldots + \lambda_k z_k$ is a convex combination in $R$.  If we look at the last row of this linear combination we have $\lambda_1 v_1 + \ldots + \lambda_k v_k = 0$.  This means that $\lambda_1 = \ldots = \lambda_k$, as we wanted.
\end{proof}

\section{Construction and basic properties of multi-class SVM}\label{sec:construction}

We are now ready to formalize the multiclass SVMs described in the introduction.  Given $k$ sets $A_1, \ldots, A_k$ in $\rr^d$ whose convex hulls do not all overlap, we seek a family of $k$ half-spaces $H_1,\ldots, H_k$ such that $A_i \subset H_i$ for each $i=1,\ldots, k$ and so that the half-space $H_1, \ldots, H_k$ do not all intersect.  For the following definition we need the subspaces $U_i = \{\bar{x} \otimes v_i : x \in \rr^d\}$ and their associated functions $f_i$ defined above.

\begin{definition}[Simple TSVM]\label{def:simpleTSVM}
Let $A_1, \ldots, A_k$ be finite families of points in $\rr^d$ whose convex hulls do not intersect.  We define the multi-class support vector machine (Simple TSVM) as a family of $k$ closed half-spaces $H_1, \ldots, H_k$ obtained as follows.  First, for each $x \in A_i$ construct the point $S(x) = \bar{x} v_i^T \in \rr^{(d+1)(k-1)}$.  Let $Y$ be the collection of all points obtained this way.  Find $H$ the closed half-space in $\rr^{(d+1)(k-1)}$ that contains $Y$ and whose distance from the origin is maximal.  For $i=1,\ldots, k$ the half-space $H_i$ is defined as $H_i = f_i(U_i \cap H)$.
\end{definition}

The computation of (simple TSVM) consists of finding the distance from $Y$ to the origin.  We can also think of this as finding the largest-margin SVM in $\rr^{(d+1)(k-1)}$ that separates the origin from $Y$, and then doubling the distance to the origin.  The discussion in the previous section shows that this multi-class support vector machine satisfies the desired properties.  For a \textbf{soft-margin} version, it suffices to compute in $\rr^{(d+1)(k-1)}$ an SVM with one class equal to $Y$ and the other equal to $\{0\}$.  If we denote by $\tau(a,b;d)$ the complexity of an algorithm to compute an SVM with data points in $\rr^d$ and two classes of size $a$ and $b$, then the complexity of computing (simple TSVM) is $\tau(n,1;(d+1)(k-1))$, where $n$ is the number $|A_1|+ \dots +|A_k|$ of data points.  Any other performance metrics we have for an SVM transfer to (simple TSVM) if we do the change of parameter as outlined above.

For the second type of multi-class SVM, we consider the following alternative definition.  Recall that in the space of $(d+1)\times(k-1)$ matrices we denoted by $R$ the subspace where the last row is equal to zero.

\begin{definition}[TSVM]\label{def:TSVM}
Let $A_1, \ldots, A_k$ be finite families of points in $\rr^d$ whose convex hulls do not intersect.  We define the multi-class support vector machine (TSVM) as a family of $k$ closed half-spaces $H_1, \ldots, H_k$ obtained as follows.  First, for each $x \in A_i$ construct the point $S(x) = \bar{x} v_i^T \in \rr^{(d+1)(k-1)}$.  Let $Y$ be the collection of all points obtained this way and consider $P = R \cap \conv(Y)$.  Compute $p$ the closest point of $P$ to the origin, and $H$ the closed half-space in $\rr^{(d+1)(k-1)}$ that contains $Y$, whose boundary hyperplane contains $p$, and whose distance from the origin is maximal.  The half-spaces $H_i$ are defined as $H_i = f_i(U_i \cap H)$.
\end{definition}

Even though this definition is more involved it has two big advantages.  First, it is stable under translations of the sets of points.  Second, in the case $k=2$ it is precisely a largest-margin SVM.  We prove these properties in the next section.  Just like SVM have critical points, any (TSVM) is fixed by a small set of points.

\begin{theorem}\label{thm:simple-TSVM-support}
Let $A_1, \ldots, A_k$ be $k$ finite sets in $\rr^d$ such that $\bigcap_{i=1}^k\conv A_i = \emptyset$.  We can find subsets $A_1' \subset A_1, \ldots, A_k'\subset A_k$ such that $A_1', \ldots, A'_k$ induces the same (simple TSVM) as $A_1, \ldots, A_k$ and such that $|A_1'|+\ldots+|A_k'|\le (d+1)(k-1)$
\end{theorem}

\begin{proof}
We follow the construction in \cref{def:simpleTSVM}. Since $0 \not\in \conv(Y)$ the closest point to the origin in $\conv(Y)$ must be in a face of the polytope $\conv(Y)$.  This face $K$ can have dimension at most $(d+1)(k-1)-1$.  By Carath\'eodory's theorem, we can choose a set of at most $(d+1)(k-1)$ point in $Y \cap K$ whose convex hull contains $p$.  The subsets of $A_1,\ldots, A_k$ that induced this subset in $Y \cap K$ satisfy the condition we wanted.
\end{proof}

\begin{theorem}\label{thm:TSVM-support}
Let $A_1, \ldots, A_k$ be $k$ finite sets in $\rr^d$ such that $\bigcap_{i=1}^k\conv A_i = \emptyset$.  We can find subsets $A_1' \subset A_1, \ldots, A_k'\subset A_k$ such that $A_1', \ldots, A'_k$ induces the same (TSVM) as $A_1, \ldots, A_k$ and such that $|A_1'|+\ldots+|A_k'|\le (d+1)(k-1)$
\end{theorem}

\begin{proof}
The proof is similar to the previous theorem.  If we look for the minimal face of $\conv Y$ sustaining $p$, it has dimension at most $(k-1)(d+1)-1$, so the same application of Carath\'eodory's theorem yields the result.  The only additional detail to check is that $R \cap \conv(Y) \neq \emptyset$.  This holds because for every choice $x_1 \in A_1, \ldots, x_k \in A_k$, the baryceneter of the point $S(x_1),\ldots, S(x_k)$ is in $\conv(Y) \cap R$. 
\end{proof}

We denote the subsets obtained by Theorem \ref{thm:simple-TSVM-support} and Theorem \ref{thm:TSVM-support} as the \textit{support vectors} of a (simple TSVM) or (TSVM), respectively.

As mentioned above, to compute (Simple TSVM) we need to compute an SVM in a $(k-1)(d+1)$-dimensional space with $|A_1|+\ldots + |A_k|+1$ points.  A direct approach to compute (TSVM) would be to first find the vertices of $\conv(Y)\cap R$ and solve the induced SVM.  We know $R$ is a linear subspace of co-dimension $k-1$, so the vertices of $\conv(Y) \cap R$ should be the intersection of the $(k-1)$-skeleton of $\conv(Y)$ with $R$.  Due to Lemma \ref{lem:barycenter}, this is a subset of the barycenters of $k$-tuples with one element in each $Y_i$.  Therefore, we can compute these barycenters and then compute an SVM in $R$.  This leads us to solve an SVM in a $(k-1)d$-dimensional space with $|A_1|\cdot \ldots \cdot |A_k|+1$ points.

Theorem \ref{thm:TSVM-support} shows that computing a TSVM can be treated as a linear programming type problem, as in the framework of Sharir and Welzl \cites{Sharir:1992ih}.  This is a randomized approach to problems which are combinatorially similar to linear programming problems, so that they can be solved in expected linear time in the input, which is a signficant reduction over brute-force approaches.  This means that for fixed $k, d$ we can compute (TSVM) with a randomized algorithm in expected time linear in $|A_1|+\ldots +|A_k|$. We describe the process in Algorithm \ref{alg:euclid}, before translating back to $\rr^d$.

The key idea to compute this is to order the points randomly.  At any point, we have computed (TSVM) for the first $t-1$ points and we keept track of the support vector of this TSVM.  When including the $t$-th point, if we don't need to adjust the current halfspace $H$ generated by (TSVM), we keep goint.  Otherwise, we adjust our guess for the support vectors and run the algorithm again for the first $t$ points.  The computations of Sharir and Welzl bound the expected number of times we need to rerun this procedure, and end up with an expected running time linear on the input.  For deeper explanations, we recommend references on linear-programming type algorithms and violator spaces \cites{Gartner:2008bp, Amenta:2017ed}.

\begin{algorithm}
\caption{Computing TSVM}\label{alg:euclid}
\begin{algorithmic}[1]
\Procedure{TSVM}{Family $Y$, Tuple $Y'$} 
\State Order $Y$ randomly as $p_1, \ldots, p_{|Y|}$ where the first $|Y'|$ elements are $Y'$.  The tuple $Y'$ must have $(k-1)(d+1)$ points, and are the candidates for the support vectors of the TSVM.
\State Find the TSVM for $Y'$, denoted $H$.  This is a half-space in $\rr^{(d+1)(k-1)}$ that does not contain the origin.
\For{each $p_t \in Y$}
    \State Check $p_t \in H$.
    \If{$p_t \not \in H$}
        \State Find the TSVM $H'$ for $Y' \cup \{p_t\}$. 
        \State Let $Y''$ be the $(d+1)(k-1)$-tuple whose TSVM is $H'$.
        \State $H$ = TSVM($\{ p_1, \hdots, p_t \}, Y'')$
        \EndIf
    \EndFor
\State \Return $H$
\EndProcedure
\end{algorithmic}
\end{algorithm}

The model (TSVM) generalizes largest-margin SVMs when $k=2$.  This is the main motivation to use the subspace $R$ in the computation.  Let us prove that this is indeed the case.

\begin{theorem}
For $k=2$, the multiclass SVM (TSVM) gives the two support hyperplanes of the largest-margin SVM of $A_1$ and $A_2$.
\end{theorem}

\begin{proof}
Notice that $k=2$ is the only case when $U_i = R_i$ for all values of $i$.  Additionally, each $U_i$ is a translate of $R$.  In this case we also have $v_1 = 1, v_2 = -1$ in $\rr^1$.  Therefore $R_1 = \{\bar{x} \in \rr^{d+1}: x \in \rr^d\}$ and $R_2 = \{-\bar{x} \in \rr^{d+1}: x \in \rr^d\}$.  Let $p$ be the closest point of $\conv(Y)\cap R$ to the origin, and $H$ be the affine hyperplane in $\rr^{d+1}$ through $p$ from \cref{def:TSVM}.  Let $H'$ be the hyperplane through the origin of $\rr^{d+1}$, parallel to $H$, and let $H'' = H \cap U_1$.  
Notice that $\|p\|$ is the distance between $H$ and $H'$.  Since a translate of $p$ lies in $U_1$, $H''+p$ contains the support vectors in of $A_1$ in $U_1$.  The same holds for $(H'\cap U_2)+p$ for the support vectors in $A_2$, so $H''-p$ contains the support vectors of $A_2$ in $U_1$.  This means that the (TSVM) induced by $A_1, A_2$ is an SVM at common distance $\|p\|$ from each side.

Similarly, given a separating hyperplane $\tilde{H}$ for $A_1, A_2$ at distance $\varepsilon$ from each set, we can embed $\rr^d$ in $U_1$ and then reflect the embedding of $A_2$ with respect to the origin in $\rr^{d+1}$ so that it lies in $U_2$.  If we extend $\tilde{H}$ through the origin in $\rr^{d+1}$, we have a hyperplane through the origin at distance $\varepsilon$ from the convex hull of the embedding of $A_1$ in $U_1$ and $A_2$ in $U_2$.  The largest margin SVM must therefore coincide with the one induced by (TSVM).
\end{proof}

An illustration of the ideas behind this proof is shown in \cref{fig:margin}.

\section{Subdivision of ambient space and potential classification errors}\label{sec:subdivision}

In each of \cref{def:simpleTSVM} and \cref{def:TSVM} we use a half-space $H$ in $\rr^{(d+1)(k-1)}$ that does not contain the origin to generate the corresponding half-space $H_1, \ldots, H_k$ in $\rr^d$.

In each case, we can introduce a half-space $H'$ that is a translate of $H$ and whose boundary contains the origin.  Notice that the half-spaces $H'_i = f_i(U_i \cap H')$ for $i=1,\ldots k$ have non-empty intersection but their interiors have an empty intersection.  This is a direct consequence of \cref{lemma:sarkaria} because $H'$ contains the origin and the interior of $H'$ does not.

As an illustration, for $k=2$ the two half-space $H'_1, H'_2$ from (TSVM) share their boundary, which is precisely the largest-margin SVM for $A_1, A_2$.

Let $T= \bigcap_{i=1}^{k} H'_k$ and $\Delta = \rr^d\setminus \left( \bigcup_{i=1}^k H_i^{\circ}\right)$, where $X^{\circ}$ denotes the interior of $X$ for any $X \subset \rr^d$.  Now for $i=1,\ldots, k$ we define the convex sets
$M_i = \{p + tq: p \in T, \ t \ge 0, \ q \in \Delta \cap H_i\}$.


\begin{figure}
\floatbox[{\capbeside\thisfloatsetup{capbesideposition={right,top},capbesidewidth=8cm}}]{figure}[\FBwidth]
{\caption{This figure shows the process to find (TSVM) for two sets of points.  First we embed the sets in $U_1$, then we reflect $A_2$ across the origin to obtain their representatives in $U_2$.  We take the convex hull of $Y_1$ and $Y_2$ and intersect it with $R$, which in the figure gives us a hexagon.  We take the closest point $p$ to the origin in $\conv(Y)\cap R$ and construct a hyperplane parallel to the facet containing $p$ of $\conv(Y)$ through the origin.  This hyperplane intersects $U_1$ in the largest-margin SVM for the original sets.}\label{fig:margin}}
{\includegraphics[width=0.4\textwidth]{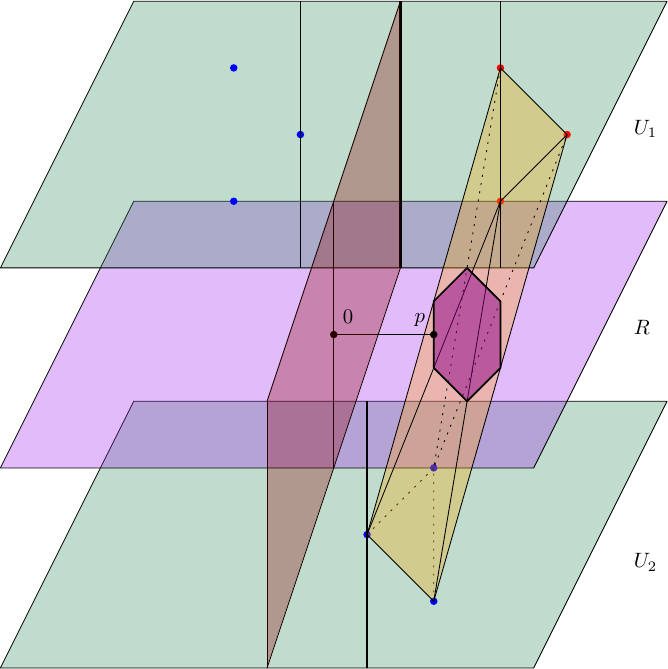}}
\end{figure}

Intuitively, $\Delta$ is the polytopal region not contained in the union of the $H_i$.  The set $T$ is an affine subspace inside $\Delta$.  If $k \le d+1$, we the set $\Delta$ is constructed by making a simplex in the orthogonal complement of $T$ and extending it in the directions of $T$.  The set $M_i$ is formed by taking all possible rays that start at $T$ and go in the direction of a point of $H_i$ in the boundary of $\Delta$.  The case when $\Delta$ is a simplex is perhaps the most illustrative one, since in this case $T$ is a point and we simply take the cones from $T$ towards each of the facets of $\Delta$.  This case looks like \cref{fig:example-full} (3).

As mentioned before, the condition needed to generate (TSVM) or (simple TSVM) is that the convex hulls of the sets $A_i$ do not all overlap.  If the convex hulls of fewer of these sets overalp, any model that subdivides $\rr^d$ into convex pieces is bound to miss-label some data.  We minimize the misslabelings with our constructions.  
 
\section{Equivariance}\label{sec:equivariance}

In this section we describe how the multi-class SVMs we introduced interact with transformations of the set of points.  It is clear that if we apply the same affine transformation to the sets of points $A_1, \ldots, A_r$ and the half-spaces $H_1, \ldots, H_r$ the containments are preserved, but we are interested to see if the algorithms to obtain $H_1, \ldots, H_r$ behave as expected with these transformations.

\begin{theorem}\label{thm:equivariance-simple}
Let $M$ be an orthogonal linear transformation of $\rr^d$.  Let $H_1, \ldots, H_r$ be the (simple TSVM) induced by $A_1, \ldots, A_r$.  Then $(M H_1, \ldots, MH_r)$ is the (simple TSVM) induced by $MA_1, \ldots, MA_r$.
\end{theorem}

\begin{proof}
First notice that $M$ can be extended to $\rr^{d+1}$ by acting on the first $d$ coordinates and leaving the last coordinate fixed.  This is also an orthogonal transformation.  We denote this transformation by $M_2$, so $\overline{Mx} = M_2\bar{x}$.  Finally, we denote by $M_3$ the transformation on $\rr^{(d+1)(k-1)}$ that multiplies every column of a $(d+1)\times(k-1)$ matrix by $M_2$, so $y \mapsto M_2y$ as a product of matrices.

This last transformation is also orthogonal.  To see this, we first show that it preserves the dot product between vectors in $U_i$ and $U_j$ for any (possibly equal) $i$ and $j$.  We use a known factorization for the dot product of tensor products, as shown below.

\begin{align*}
\langle \bar{x} \otimes v_i, \bar{y} \otimes v_j\rangle & = \langle \bar{x}, \bar{y} \rangle \langle v_i, v_j \rangle = \langle M_2 \bar{x}, M_2 \bar{y} \rangle \langle v_i, v_j \rangle = \\ & = \langle (M_2 \bar{x}) \otimes v_i, (M_2 \bar{y}) \otimes v_j\rangle = \langle M_3 (\bar{x} \otimes v_i) , M_3 (\bar{y} \otimes v_j) \rangle    
\end{align*}

Consider the union of an affine basis for each of $U_1, \ldots, U_{k-1}$.  This set of $(d+1)(k-1)$ vectors forms a basis of $\rr^{(d+1)(k-1)}$, and $M_3$ preserves the dot product between any two of these vectors.  Therefore $M_3$ preserves the dot product in $\rr^{(d+1)(k-1)}$ and is therefore orthogonal.

\end{proof}

\begin{theorem}
Let $M$ be an orthogonal linear transformation of $\rr^d$ and $H_1, \dots, H_k$ be the (TSVM) induced by $A_1, \ldots, A_k$.  Then $(M H_1, \ldots, MH_k)$ is the (TSVM) induced by $MA_1, \ldots, MA_k$.
\end{theorem}

\begin{proof}
We follow the ideas used in the proof of \cref{thm:equivariance-simple}.  We notice that $M_3$ fixes $R$.  Therefore, the restriction of $M_3$ to $R$ is an orthogonal transformation.  This means that for any half-space $H$ in $\rr^{(d+1)(k-1)}$, we have $(M_3H) \cap R = M_3(H\cap R)$.  Again, if $H$ is the half-space in $\rr^{(d+1)(k-1)}$ that induces our (TSVM), we have that $M_3H$ is the half-space for the new set of points. 

Now, if we consider the (TSVM) induced by $A_1, \ldots, A_k$, we have to find the half-space $H$ in $\rr^{(d+1)(k-1)}$ farthest from the origin that contains $Y$.  Clearly, $M_3H$ is the farthest half-space from the origin that contains $M_3Y$.  For $i=1,\ldots, k$, we also have $f_i(U_i \cap (M_3 H)) = M f_i(U_i \cap H)$.
\end{proof}

\begin{theorem}
Let $q$ be a vector in $\rr^d$.  Let $X$ be the set of support vectors of the (TSVM) induced by $A_1, \ldots, A_k$.  Then $X+q$ is the set of support vectors of the (TSVM) induced by $A_1+q, \ldots, A_k+q$.
\end{theorem}

\begin{proof}
To find the (TSVM) incuded by $A_1, \ldots, A_k$ we need to compute $\conv(Y) \cap R$.  Notice that this set is invariant under translations of $A = \bigcup_{i=1}^k A_i$.  This is because for any points $x_1, \ldots, x_k$ the barycenter of $\{\bar{x}_1 \otimes v_1, \ldots, \bar{x}_k \otimes v_k\}$ is the same as the barycenter of $\{\overline{(x_1+q)} \otimes v_1, \ldots, \overline{(x_k+q)} \otimes v_k\}$.  Since $\conv(Y) \cap R$ does not change, the set of support vectors remains the same.
\end{proof}

\section{Acknowledgments}
The author thanks Henry Adams for helpful comments.



\begin{bibdiv}
\begin{biblist}

\bib{Amenta:2017ed}{inproceedings}{
      author={Amenta, N.},
      author={De~Loera, J.},
      author={Sober\'on, P.},
       title={{H}elly's theorem: {N}ew variations and applications},
        date={2017},
   booktitle={Algebraic and geometric methods in discrete mathematics},
      editor={Harrington, Heather},
      editor={Omar, Mohamed},
      editor={Wright, Matthew},
   publisher={American Mathematical Society},
       pages={55\ndash 96},
}

\bib{Adams2022}{article}{
      author={Adams, Henry},
      author={Farnell, Elin},
      author={Story, Brittany},
       title={Support vector machines and {R}adon's theorem},
        date={2022},
        ISSN={2639-8001},
     journal={Found. Data Sci.},
      volume={4},
      number={4},
       pages={467\ndash 494},
         url={https://doi.org/10.3934/fods.2022017},
}

\bib{Barany:2016vx}{article}{
      author={Bárány, Imre},
      author={Blagojević, Pavle V.~M.},
      author={Ziegler, G\"unter~M.},
       title={{Tverberg's Theorem at 50: Extensions and Counterexamples}},
        date={2016},
     journal={Notices of the American Mathematical Society},
      volume={63},
       pages={732 \ndash  739},
}

\bib{Boser1992}{inproceedings}{
      author={Boser, Bernhard~E.},
      author={Guyon, Isabelle~M.},
      author={Vapnik, Vladimir~N.},
       title={A training algorithm for optimal margin classifiers},
        date={1992},
   booktitle={Proceedings of the 5th annual acm workshop on computational
  learning theory},
       pages={144\ndash 152},
}

\bib{Barany:1996bz}{incollection}{
      author={Bárány, Imre},
      author={Onn, Shmuel},
       title={{Colourful linear programming}},
    language={English},
        date={1996},
      series={Integer Programming and Combinatorial Optimization},
      volume={1084},
   publisher={Springer Berlin Heidelberg},
       pages={1 \ndash  15},
}

\bib{Barany:2018fy}{article}{
      author={Bárány, Imre},
      author={Soberón, Pablo},
       title={{Tverberg's theorem is 50 years old: A survey}},
        date={2018},
     journal={Bulletin of the American Mathematical Society},
      volume={55},
      number={4},
       pages={459 \ndash  492},
}

\bib{Blagojevic:2017bl}{incollection}{
      author={Blagojevi\'{c}, Pavle V.~M.},
      author={Ziegler, G\"unter~M.},
       title={{Beyond the Borsuk--Ulam Theorem: The Topological Tverberg
  Story}},
        date={2017},
      series={A Journey Through Discrete Mathematics},
      volume={34},
   publisher={Springer, Cham},
       pages={273 \ndash  341},
}

\bib{crammer2001algorithmic}{article}{
      author={Crammer, Koby},
      author={Singer, Yoram},
       title={On the algorithmic implementation of multiclass kernel-based
  vector machines},
        date={2001},
     journal={Journal of machine learning research},
      volume={2},
      number={Dec},
       pages={265\ndash 292},
}

\bib{duan2005best}{inproceedings}{
      author={Duan, Kai-Bo},
      author={Keerthi, S~Sathiya},
       title={Which is the best multiclass svm method? an empirical study},
organization={Springer},
        date={2005},
   booktitle={International workshop on multiple classifier systems},
       pages={278\ndash 285},
}

\bib{de2020stochastic}{article}{
      author={De~Loera, Jesus~A},
      author={Hogan, Thomas},
       title={Stochastic tverberg theorems with applications in multiclass
  logistic regression, separability, and centerpoints of data},
        date={2020},
     journal={SIAM Journal on Mathematics of Data Science},
      volume={2},
      number={4},
       pages={1151\ndash 1166},
}

\bib{Franc.2002}{article}{
      author={Franc, Voit\v{c}ch},
      author={Hlav\'a\v{c}, V\'aclav},
       title={{Multi-class Support Vector Machine}},
        date={2002},
     journal={Object recognition supported by user interaction for service
  robots},
      volume={2},
       pages={236\ndash 239},
}

\bib{Gartner:2008bp}{article}{
      author={G\"artner, Bernd},
      author={Matou\v{s}ek, Ji\v{r}\'i},
      author={R{\"u}st, Leo},
      author={{\v{S}}kovro\v{n}, Petr},
       title={{Violator spaces: Structure and algorithms}},
        date={2008},
     journal={Discrete Applied Mathematics},
      volume={156},
      number={11},
       pages={2124\ndash 2141},
}

\bib{hearst1998support}{article}{
      author={Hearst, Marti~A.},
      author={Dumais, Susan~T},
      author={Osuna, Edgar},
      author={Platt, John},
      author={Scholkopf, Bernhard},
       title={Support vector machines},
        date={1998},
     journal={IEEE Intelligent Systems and their applications},
      volume={13},
      number={4},
       pages={18\ndash 28},
}

\bib{hsu2002comparison}{article}{
      author={Hsu, Chih-Wei},
      author={Lin, Chih-Jen},
       title={A comparison of methods for multiclass support vector machines},
        date={2002},
     journal={IEEE transactions on Neural Networks},
      volume={13},
      number={2},
       pages={415\ndash 425},
}

\bib{har2020journey}{article}{
      author={Har-Peled, Sariel},
      author={Jones, Mitchell},
       title={Journey to the center of the point set},
        date={2020},
     journal={ACM Transactions on Algorithms (TALG)},
      volume={17},
      number={1},
       pages={1\ndash 21},
}

\bib{Radon:1921vh}{article}{
      author={Radon, Johann},
       title={{Mengen konvexer Körper, die einen gemeinsamen Punkt
  enthalten}},
        date={1921},
     journal={Mathematische Annalen},
      volume={83},
      number={1},
       pages={113 \ndash  115},
}

\bib{Sarkaria:1992vt}{article}{
      author={Sarkaria, Karanbir~S.},
       title={{Tverberg’s theorem via number fields}},
        date={1992},
     journal={Israel journal of mathematics},
      volume={79},
      number={2},
       pages={317 \ndash  320},
}

\bib{steinwart2008support}{book}{
      author={Steinwart, Ingo},
      author={Christmann, Andreas},
       title={Support vector machines},
   publisher={Springer Science \& Business Media},
        date={2008},
}

\bib{soberon2015equal}{article}{
      author={Sober{\'o}n, Pablo},
       title={Equal coefficients and tolerance in coloured tverberg
  partitions},
        date={2015},
     journal={Combinatorica},
      volume={35},
      number={2},
       pages={235\ndash 252},
}

\bib{Sarkar:2020uk}{article}{
      author={Sarkar, Sherry},
      author={Sober\'{o}n, Pablo},
       title={Tolerance for colorful {T}verberg partitions},
        date={2022},
     journal={European J. Combin.},
      volume={103},
       pages={Paper No. 103527},
}

\bib{Sharir:1992ih}{inproceedings}{
      author={Sharir, Micha},
      author={Welzl, Emo},
       title={{A combinatorial bound for linear programming and related
  problems}},
        date={1992},
      series={Annual Symposium on Theoretical Aspects of Computer Science},
      volume={577},
       pages={567 \ndash  579},
}

\bib{Tverberg:1966tb}{article}{
      author={Tverberg, Helge},
       title={{A generalization of Radon’s theorem}},
        date={1966},
     journal={J. London Math. Soc},
      volume={41},
      number={1},
       pages={123 \ndash  128},
}

\bib{Veelaert.2015}{article}{
      author={Veelaert, Peter},
       title={{Combinatorial properties of support vectors of separating
  hyperplanes}},
        date={2015},
        ISSN={1012-2443},
     journal={Annals of Mathematics and Artificial Intelligence},
      volume={75},
      number={1-2},
       pages={89\ndash 115},
}

\end{biblist}
\end{bibdiv}
\end{document}